\documentclass{article}
\usepackage{cite}
\usepackage{amsmath,amssymb,amsfonts}
\usepackage{graphicx}
\usepackage{authblk}
\newtheorem{theorem}{Theorem}
\newtheorem{lemma}{Lemma}
\usepackage[ruled,vlined]{algorithm2e}
\newenvironment{proof}{\paragraph{Proof}}{\hfill$\square$}

\begin{document}

\title{Learning from Satisfying Assignments Using Risk Minimization}

\author[1]{Manjish Pal\thanks{manjishster@gmail.com}}
\author[2]{Subham Pokhriyal\thanks{subham0100@gmail.com}}
\affil[1]{IIT-Kharagpur, India}
\affil[2]{Chang Gung University, Taiwan}

\maketitle

\begin{abstract}
In this paper we consider the problem of Learning from Satisfying Assignments introduced by \cite{1} of finding a distribution that is a close approximation to the uniform distribution over the satisfying assignments of a low complexity Boolean function $f$. In a later work \cite{2} consider the same problem but with the knowledge of some continuous distribution $D$ and the objective being to estimate $D_f$, which is $D$ restricted to the satisfying assignments of an unknown Boolean function $f$. We consider these problems from the point of view of parameter estimation techniques in statistical machine learning and prove similar results that are based on standard optimization algorithms for Risk Minimization. 

\end{abstract}

\section{Introduction}
\emph{Computational Learning Theory} is a field of theoretical computer science that mostly deals with computational complexity theoretic aspects of machine learning with its origins from works of Valiant \cite{5} that introduced the notion of Probably Approximately Correct (PAC) learnability. Learning of Boolean Functions has become an important part of this topic with a variety of results obtained over the past three decades or so. We consider the problem mentioned in the previous sections of estimating the uniform distribution over the satisfying assignments of an unknown Boolean function $f$ from the point of view of standard parameter estimation techniques in Statistical Machine Learning which are highly well established but not much used in the context of Boolean function analysis. An extensive motivation and background for the problem is mentioned in \cite{1,2}. The problem has been considered in another set up in which we assume that $D$ is a known continuous distribution like Gaussian, log-concave etc. and we are given i.i.d. samples from $D_f$, which is defined as $D$ restricted to satisfying assignments of $f$ where $f$ is again an unknown Boolean function, and we intend to estimate $D_f$. In these works the authors prove both algorithmic and impossibility results depending on the nature of $f$ (for eg. $f$ being a threshold function w.r.t. a low degree polynomial). In this paper we restrict our attention to the original version of the problem and hope that our ideas would translate to other case as well.\\ 
Parameter estimation is a widely explored topic in machine learning with a variety of techniques which have been developed over several years. The main goal in this exploration is to approximate certain unknown distribution (from which the data is assumed to be coming) by a distribution that depend on certain parameters $\theta$; and use certain optimization procedures to find the values of the parameters $\theta$. In this paper we consider the same problem of learning the distribution of satisfying assignments but we use standard parameter estimation techniques for solving this problem. More formally, for an unknown distribution $P(x)$ chosen from an unknown distribution $D$ we approximate it with $P_{\theta}(x)$ and compute the parameters $\theta$ by maximizing the log likelihood function. It turns out that this optimization is equivalent to minimizing the expected log loss, which in turn is Risk Minimization under the distribution $D$. We use these relations along with bounds that relate KL-divergence and $l_1$ distance to get results similar and comparable with De et al. 
\section{Risk Minimization}
Risk Minimization \cite{4} is a statistical machine learning framework in which we formally define the notion of error of a classifier. The risk minimization can be defined by either w.r.t the given data i.e. Empirical Risk Minimization (ERM) or it can be defined w.r.t the distribution from which the data is assumed to be coming from. More formally, ERM for a hypothesis $h$ and a dataset $S$ with $|S|= m$ is defined as 
\begin{center}
    $L_{S}(h) = |\{i \in [m]: h(x_i) \neq y_i\}|/m$
\end{center}
whereas the risk w.r.t. the distribution $D$ the risk is given as
\begin{center}
    $L_{S}(f,h) = \mathbb{P}_{x \sim D}[f(x) \neq h(x)]$ 
\end{center}
We will use these notions of risk minimization to revisit some already established relation between Maximum Likelihood Estimator (MLE) and Empirical Risk Minimization when we are dealing with a particular loss function. In general when we have the loss function $l(\textbf{w},x)$ as 
\begin{center}
    $l(\textbf{w},x) = -\log (P_{\textbf{w}}[x])$
\end{center}
where $P_w[x]$ is the parametric estimate of the true and unknown $P[x]$ (the probability of choosing $x$ from the unknown distribution $D$) and $\textbf{w}$ is the unknown parameter to be estimated. It is known that according to this choice of loss function, the MLE is same as the Risk Minimization i.e.
\begin{eqnarray*}
E[l(\textbf{w},x)] &=& - \large \sum_x P[x] \log P_{\textbf{w}}[x] \\
                &=&  \underbrace{\large \sum_x P[x] \log \frac{P[x]}{P_{\textbf{w}}[x]}} + \large \sum_x P[x] \log \frac{1}{P[x]} 
\end{eqnarray*}
We notice that the term in the above equation that is in braces is the \emph{Kullback Leibler Divergence} of the two distributions $P[x]$ and $P_{\textbf{w}}[x]$ which we will denote by $KL(P,P_{\textbf{w}})$. Intuitively, minimizing the term $E[l(\textbf{w},x)]$ w.r.t. $\textbf{w}$ leads to the KL divergence being very small because the other term is independent of $\textbf{w}$. There are two approaches that can be followed to minimize $E[l(\textbf{w},x)]$: first is to use the celebrated \emph{Expectation Maximization Algorithm} which is not of much use for us directly because the algorithm is well known to get stuck at local optima. Instead we prefer the second approach that is based on using the Stochastic Gradient Descent algorithm to serve our purpose that is guaranteed to return a value of $\textbf{w}$ that is close to the global optimum. In what follows we describe the SGD algorithm and later describe how these results can be used to get bounds similar to prior works in the context of learning from the satisfying assignments of a Boolean function.  

\section{Stochastic Gradient Descent}
Over the past few years, many learning algorithms for minimizing risk functions have been proposed in the field of convex optimization. One of the important algorithms for convex optimization that is widely used in machine learning is called Stochastic Gradient Descent (SGD). 
In SGD we try to minimize risk function $L_D(\textbf{w})$ in which we are not aware of the unknown distribution $D$ and hence we can't directly compute gradient for $L_D(w)$ that is needed for standard gradient descent.
The approach of SGD is to initialize the gradient in random direction and use update rule of gradient descent algorithm for optimization. A major component in the SGD algorithm is the notion of the set of  \emph{subgradients} that is denoted by $\partial l(\textbf{w},z)$ that is mentioned in the following vanilla version of SGD. For a detailed analysis of SGD algorithm in the context of machine learning can be obtained from \cite{4}.

\begin{algorithm}

\SetAlgoLined

\textbf{Input:} Scalar $ \eta>0 $, integer $T>0$ \;
\textbf{Initialization:} $\textbf{w}^{(1)} = 0$\;
\For{$t = 1;\ t \leq T;\ t = t + 1$}{
  sample $x \sim D$  %\tcc{this step picks an element from given data}
  pick $v_t \in \partial l(\textbf{w}^t,x)$ %\tcc{this step uses randomization.}
  \textbf{update} $\textbf{w}^{(t+1)} = \textbf{w}^{(t)} - \eta v_t$ \;
  }

\textbf{Output} $\bar{\textbf{w}}=\frac{1}{T} \sum_{t=1}^T \textbf{w}^{(t)} $\;
 
\caption{SGD($X,\eta,T$)} 
\end{algorithm}

\begin{theorem} Consider a convex-Lipschitz-bounded learning problem with parameters $ \rho $, B. Then, for every $ \epsilon > 0$, if we run the SGD method for minimizing
$L_D{(\textbf{w})}$ with a number of iterations (i.e., number of examples)

\begin{center} 
\Large{$T \geq \frac{B^2\rho^2}{\epsilon^2}$}  
\end{center}

and with $\eta =\sqrt{\frac{B^2\rho^2}{\epsilon^2}}$, then the output of SGD satisfies

\begin{center}
$E[L_D(\bar{\textbf{w}})] \leq \min\limits_{\textbf{w} \in H} L_D(\textbf{w}) +\epsilon$.
\end{center}
\end{theorem}

\section{Boolean Functions and Main Result}
In this section we describe how our results relate with the results in the context of the problem of learning from satisfying assignments of a Boolean Function. In the Boolean Function setting we are given satisfying assignments of an unknown Boolean function $f:\{-1,1\}^n \rightarrow \{-1,1\}$ and our objective is to output a highly accurate estimate of $U_{f^{-1}(1)}$. In a different version of the problem one is given a continuous distribution $D$ and i.i.d. samples of $D_f$ where $D_f$ is obtained by restricting $D$ to satisfying assignments of an unknown Boolean function $f$. The objective again is to get an accurate estimate of $D_f$. In those papers the authors measure the distance of the estimated distribution and the actual distribution by \emph{variation distance} that is also the $l_1$-norm of the two probability distributions. In this paper we first estimate there distance by KL divergence and then use the distortions bounds to translate the error bounds to the $l_1$-norm. 

\begin{lemma}
If $\bar{\textbf{w}}$ is the value returned by the algorithm then
\begin{center} 
$KL(P,P_{\bar{\textbf{w}}}) \leq \epsilon \mbox{ and } l_1(P,P_{\bar{\textbf{w}}}) \leq \sqrt{2\epsilon}$
\end{center}
\end{lemma}
\begin{proof}
Let $\min\limits_{\textbf{w} \in H} L_D(\textbf{w}) = A(\textbf{w}')$ and $\Delta =\sum_x P[x] \log \frac{1}{P[x]}$, Clearly $KL(P,P_{\textbf{w}'}) = 0$. Thus $L_D(\bar{\textbf{w}}) =  KL(P,P_{\bar{\textbf{w}}}) + \Delta$. Therefore, 
\begin{center}
$|A(\textbf{w}') - L_D(\bar{\textbf{w}})| = |KL(P,P_{\textbf{w}'}) - KL(P,P_{\bar{\textbf{w}}})|$
\end{center}
From Theorem 1 we know that $|A(\textbf{w}') - L_D(\bar{\textbf{w}})| \leq \epsilon$. Hence, $KL(P,P_{\bar{\textbf{w}}}) \leq \epsilon $. From Pinsker Inequality we have $l_1^2(P,P_{\bar{\textbf{w}}})\leq 2KL(P,P_{\bar{\textbf{w}}})$.
\begin{center}
    $l_(P,P_{\bar{\textbf{w}}}) \leq \sqrt{2\epsilon}$
\end{center}
\end{proof}

Thus we have the following result,st $1-\delta$ a probability distribution $P_{\textbf{w}}$ that is close to uniform distribution over $f^{-1}(1)$ by a variation distance of $\epsilon$.

In  our  framework,  we  notice  that  for  SGD  to  be  executed  we  need  to  know  an exact  functional  form  for $L_{D}(\textbf{w}) = E[l(\textbf{w},x)]  = g(\langle \textbf{w},x \rangle)$  in  order  to  compute $\nabla(g)$  w.r.t. \textbf{w}. From the well known \emph{No  Free  Lunch  Theorem} \cite{4} we are aware that in order to achieve learnability of any target function or distribution, we need to have some prior information about the class of hypothesis we are restricting ourselves to.  In the current set up,  we incorporate this information by assuming that $E[l(\textbf{w},x)]$ can be expressed as $g(\langle \textbf{w},x \rangle)$ where $g$ is convex $\rho$-Lipschitz function. For eg. we can choose $g$ to be any of $\sin$, sigmoid, $\sqrt{x^2 + 5}$ which are 1-Lipschitz. Our assumption about $L_{D}(\textbf{w})$ as $g(\langle \textbf{w},x \rangle)$ is similar to the assumptions made by \cite{1,2} in which authors put a restrictions on $f$ as a low degree threshold function. Although we have not been able to show a mathematical equivalence between the assumptions; both the assumptions are restricting the nature of optimization problem to certain extent.

Also w.l.o.g. we will assume that the value of $B$ used in the SGD is 1 i.e. we are optimizing inside a unit sphere, $\|\textbf{w}\|=1$. 

\textbf{Retrieving $x$ given $\bar{\textbf{w}}$}: Using the ``reformulation'' of $L_{D}(\textbf{w})$ as $g(\langle \textbf{w},x \rangle)$, we are able retrieve $\bar{\textbf{w}}$ and $b = L_{D}(\bar{\textbf{w}})$ but our aim was to sample assignments $x$ that are from the desired distribution. For that we simply pick an $\bar{x} \in \{-1,1\}^n$ and compute $a = g(\langle \bar{\textbf{w}},\bar{x} \rangle)$ and compare with $b$. If $|a-b| \leq \epsilon_1$ then we declare that $\bar{x} \in f^{-1}(1)$ otherwise not.    

\begin{algorithm}
\SetAlgoLined
\textbf{Input}: $X = {x_1,x_2, \dots x_k}$ be the satisfying assignment of $f$\;
$\bar{w}=  SGD(X,B,\rho,\epsilon)$ %{This step uses $g$, a 1-Lipschitz function}
\textbf{Output} : $P_{\bar{\textbf{w}}}$ %\tcc{$x$ can be obtained from $P_{\bar{\textbf{w}}}$ }
 \caption{Algorithm A}
\end{algorithm}

\begin{theorem}
The Algorithm $A$ returns a probability distribution $ P_{\bar{\textbf{w}}} $ that is $\epsilon$ close to the unknown distribution $U_{f^{-1}(1)}$  w.r.t. variation distance with probability atleast $1-\delta$ and runs in time in $\Theta(\frac{n}{\epsilon^2} \log(\frac{1}{\delta}))$ where $n$ is the number of variables of the Boolean function $f$. 
\end{theorem}
\begin{proof}
From Theorem 1 we know that one iteration of the SGD requires $\Theta(\frac{B^2\rho^2}{\epsilon^2})$ and from the previous discussion we know that $B=1$ and we can choose an  $g(\langle \textbf{w},x \rangle)$ such that it is $\rho$-lipschitz with $\rho=1$ for example $l$ can be choosen as sigmoid, sine etc. Using Chernoff-Hoeffding \cite{4} bounds we can infer that by repeating the SGD algorithm independently $\log(\frac{1}{\delta})$ iterations, we can success probability to $1-\delta$. 

\end{proof}

\section{Conclusion}
We have shown that the Empirical Risk Minimization approach can lead us to results that are comparable \cite{2} for the problem of learning from satisfying assignments \cite{2}. The technique of machine learning models for parameter estimation helps us to estimate the uniform distribution  over $f^{-1}(1)$.
Since we know that stochastic gradient descent algorithm converges in $\Theta(\frac{n}{\epsilon^2})$ iterations, We have been able to conclude that our algorithm for estimating $U_{f^{-1}(1)}$ runs in  $\Theta(\frac{n}{\epsilon^2}\log(\frac{1}{\delta}))$ time and w.h.p. returns a probability distribution that has low variation distance with $U_{f^{-1}(1)}$.

\end{document}